\documentclass{amsart}

\usepackage{graphicx}%
\usepackage{multirow}%
\usepackage{amsmath,amssymb,amsfonts}%
\usepackage{amsthm}%
\usepackage{mathrsfs}%
\usepackage[title]{appendix}%
\usepackage{xcolor}%
\usepackage{textcomp}%
\usepackage{manyfoot}%
\usepackage{booktabs}%
\usepackage{algorithm}%
\usepackage{algorithmicx}%
\usepackage{algpseudocode}%
\usepackage{listings}%

\usepackage[colorlinks=true]{hyperref}
\usepackage{cleveref}
\usepackage[msc-links, numeric]{amsrefs}

\usepackage{booktabs}
\usepackage{subcaption}
\usepackage{tikz}
\usepackage{array}
\usepackage{url}
\usetikzlibrary{positioning}

\author[Yerrapati]{Venkata Subbaiah Yerrapati}
\address{(Yerrapati) Department of Mathematics, S. V. National Institute of Technology, Surat-7, Gujarat, India}
\email{yvsmath@gmail.com}
\thanks{The first author would like to thank the Department of Education, Government of India, for the financial assistance.}

\author[Dixit]{Rahul Dixit}
\address{(Dixit) Department of Artificial Intelligence, S. V. National Institute of Technology, Surat-7, Gujarat, India}
\email{rahuldixit@aid.svnit.ac.in}

\author[Shukla]{Ajay Kumar Shukla}
\address{(Shukla) Department of Mathematics, S. V. National Institute of Technology, Surat-7, Gujarat, India}
\email{aks@amhd.svnit.ac.in}

\theoremstyle{plain}
\newtheorem{theorem}{Theorem}[section] 
\newtheorem{lemma}[theorem]{Lemma} 

\theoremstyle{definition}
\newtheorem{definition}[theorem]{Definition} 

\theoremstyle{remark}
\newtheorem{remark}[theorem]{Remark} 

\begin{document}

    \title[Neural Ideals and Neural Codes]{Neural Ideals and Neural Codes: An Algebraic Framework for Neural Network Classification and Feature Interpretation}
    

    
    \begin{abstract}
        Understanding the features captured by the hidden layers of neural networks is a fundamental challenge in machine learning, despite their widespread success across various classification problems. In this work, we propose an algebraic framework for examining neural networks that model classification problems. Certain results, such as the correspondence between the neural network and neural ideals, algorithms for computing the neural ideals, and a stabilization theorem that enables approximation of the neural ideals, are first established. As an application to the framework, we present algorithms to identify and interpret the features captured by each hidden-layer neuron. Along with these theoretical developments, the practical performance has been demonstrated on the MNIST digit dataset, and the results highlight the pivotal role of neural ideals as a mathematical and computational tool for analyzing the features captured by neural networks. Further, we develop an interactive software that builds on the presented framework to visualize the features captured by each neuron. This tool is available at \url{https://github.com/yvs1967/neural-network-representation-explorer/}
    \end{abstract}

    
    
    \date{\today}
    \maketitle

    \section{Introduction}
    
    Artificial neural networks \cite{McCulloch1943} have emerged as one of the most successful computational paradigms for addressing complex classification problems arising in computer vision, natural language processing, scientific computing, and related areas \cite{LeCun2015}. Despite their empirical success, the mathematical understanding of the internal representations learned by such models remains limited \cite{suh2025survey}. The classification mechanism of a trained neural network is typically governed by compositions of affine transformations and nonlinear activation functions acting in high-dimensional spaces \cite{LeCun2015}. Although these architectures achieve remarkable predictive performance, the relationship between hidden-layer activations and the final classification decision is often difficult to characterize mathematically \cite{papyan2020prevalence}. Consequently, the development of rigorous mathematical frameworks for analysing the internal structure and decision-making process of neural networks has become an important research direction \cite{suh2025survey, papyan2020prevalence}.
    
    One possible approach to this problem arises from the interaction of commutative algebra, algebraic geometry, and data science. When an input is propagated through a feedforward neural network, the hidden neurons produce a pattern of activations that may be viewed as a binary firing configuration. Similar binary activation patterns appear in mathematical neuroscience, where they are studied through the theory of combinatorial neural codes. In this setting, Curto \emph{et al.} \cite{Curto2013} introduced neural rings and neural ideals as algebraic objects associated with neural codes. Given a neural code $C \subseteq \mathbb{F}_2^m$, the corresponding neural ideal $J_C$ encodes the combinatorial structure of $C$ within a Boolean polynomial ring. These algebraic constructions have proved useful in the study of biological neural systems and stimulus--response relationships. However, their application to artificial neural networks and modern classification architectures has received comparatively little attention.
    
    The objective of the present work is to develop an algebraic framework that connects neural network classification with the theory of neural ideals. We associate to each target class a collection of binary activation patterns obtained from the hidden neurons of the network and introduce the notion of an \emph{allowed code space}. This construction enables the translation of activation information into the Boolean quotient ring
    \[
    R=\mathbb{F}_2[x_1,\ldots,x_m]/\langle x_i^2-x_i \;|\; 1\leq i\leq m\rangle.
    \]
    Using pseudo-monomial generators, we construct class-specific neural ideals whose algebraic structure reflects the admissible activation patterns associated with a given classification outcome.
    
    The proposed formulation establishes a correspondence between neural network classification and polynomial ideal theory. In particular, the classification behaviour of a trained network is represented through Boolean polynomial relations determined by the associated neural ideal. We further investigate the computational aspects of constructing these ideals from finite collections of training samples and establish a stabilization result showing that the resulting ideals become invariant after the inclusion of sufficiently many samples.
    
    The main contributions of this work are summarized as follows:
    
    \begin{enumerate}
    \item We introduce the notion of an \emph{allowed code space} associated with a target class of a trained feedforward neural network.
    
    \item We adapt the theory of neural ideals to artificial neural network classification and construct class-specific neural ideals from hidden-layer activation patterns.
    
    \item We establish a classification--ideal correspondence that characterizes class membership through the generators of the associated neural ideal.
    
    \item We develop algorithms for constructing threshold maps, computing neural codes, determining allowed code spaces, and deriving neural ideals from neural network architectures.
    
    \item We prove a stabilization theorem that guarantees the eventual invariance of the neural ideals obtained from sufficiently large collections of training samples.
    
    \item We demonstrate the applicability of the proposed framework through the XOR classification problem and the MNIST handwritten digit dataset, illustrating its usefulness for analysing hidden representations and feature structures.
    
    \item We develop an interactive software framework for neural-code generation, allowed code-space construction, neural-ideal analysis, and visualization of learned neural representations.
    
    \end{enumerate}
    
    The remainder of the paper is organized as follows. Section~\ref{sec:prelim} reviews the necessary background on feedforward neural networks, Boolean polynomial rings, and neural ideals. Section~\ref{main_results} contains the principal theoretical results together with the proposed algorithms. Section~\ref{sec:applications} presents applications of the framework to the XOR classification problem and the MNIST handwritten digit dataset.

    \section{Preliminaries}\label{sec:prelim}
        In this section, we review the definitions and results required for later sections. Basic concepts are taken from \cite{zbMATH08074064}, while a detailed treatment of neural rings and neural ideals, including relevant theorems and proofs, can be found in \cite{Curto2013}.
        
        \begin{definition}[Activation Function]
            Let $n$ be a positive integer. Any function $\sigma: \mathbb{R} \to \mathbb{R}$ is called an \emph{activation function} and the associated \emph{coordinate-wise activation function} $\sigma^n : \mathbb{R}^n \to \mathbb{R}^n$ is defined by $\sigma^n(x_1, x_2, \ldots, x_n) = \left(\sigma(x_1), \sigma(x_2), \ldots, \sigma(x_n)\right)$.
        \end{definition}
        
        \begin{definition}[Feedforward Neural Network]
            Let $n_0, n_1, \ldots, n_L$ be positive integers. For each $i = 1, 2, \ldots, L$, let $A_i: \mathbb{R}^{n_{i - 1}} \to \mathbb{R}^{n_i}$ be an affine map defined by $A_i(x) = W_ix + b_i$, where $W_i \in \mathbb{R}^{n_{i}\times n_{i - 1}}$ is a weight matrix and $b_i \in \mathbb{R}^{n_i}$ is a bias vector. Let $\sigma_i: \mathbb{R}^{n_i} \to \mathbb{R}^{n_i}$ be a coordinate-wise activation function. A feedforward neural network of depth $L$ is the function $f: \mathbb{R}^{n_0} \to \mathbb{R}^{n_L}$ defined by $f(x) = (\sigma_L\circ A_L)\circ (\sigma_{L - 1}\circ A_{L - 1})\circ \cdots \circ (\sigma_1 \circ A_1)(x)$, for all $x \in \mathbb{R}^{n_0}$.
        \end{definition}
        
        \begin{definition}[Boolean Polynomial Ring]
            Let $\mathbb{F}_2$ denote the field of two elements and $\mathbb{F}_2[x_1, \ldots, x_n]$ be the polynomial ring over the field $\mathbb{F}_2$. The \emph{Boolean polynomial ring} is the quotient ring $\frac{\mathbb{F}_2[x_1, \ldots, x_n]}{\langle x_i^2 - x_i \mid \text{ for } i = 1, 2, \ldots, n\rangle}$.
        \end{definition}
        
        \begin{definition}[Neural Code]
            A \emph{neural code} $C$ is a subset of $\mathbb{F}_2^m$ and a \emph{codeword} is any element in $\mathbb{F}_2^m$.
        \end{definition}
        
        \begin{definition}[Pseudo-monomial]
             A polynomial $f \in \mathbb{F}_2[x_1, x_2, \ldots, x_m]$ is called a \emph{pseudo-monomial} if $f = \prod_{i \in \sigma}x_i\prod_{j \in \tau}(1 - x_j)$, where $\sigma, \tau \subseteq \{1, 2, \ldots, m\}$ such that $\sigma\cap\tau = \emptyset$.
        \end{definition}
        
        \begin{definition}[Characteristic Polynomial]
            Let $v = (v_1, \ldots, v_m) \in \{0,1\}^m$ be a binary vector. The \emph{characteristic polynomial} of $v$ is the pseudo-monomial $\rho_v \in \mathbb{F}_2[x_1, \ldots, x_m]$ defined by $\rho_v(x) = \prod_{v_i=1} x_i \prod_{v_j=0} (1-x_j)$.
        \end{definition}
        
        \begin{definition}[Neural Ideal]
            Let $C \subseteq \{0,1\}^m$ be a neural code. The \emph{neural ideal} associated with $C$ is the ideal 
            \[ J_C = \langle \rho_v \mid v \notin C \rangle + \langle x_i^2 - x_i \mid i=1,\ldots,m \rangle \subseteq \mathbb{F}_2[x_1,\ldots,x_m], \]
            where, for each $v=(v_1,\ldots,v_m)\in \{0,1\}^m$, $\rho_v = \prod_{v_i=1} x_i \prod_{v_j=0} (1-x_j)$ is the characteristic polynomial of $v$.
        \end{definition}
    
        \begin{lemma}[Chapter 8 \cite{zbMATH06665535}]\label{lem:des_chain_condition_initial}
            Let $\mathbb{K}$ be a finite field and $R=\mathbb{K}[x_1,\ldots,x_m]\Big/\left\langle x_i^2-x_i \mid i = 1,\ldots,m\right\rangle$. If $I_1 \supseteq I_2 \supseteq I_3 \supseteq \cdots$ is a descending chain of ideals in $R$, then there exists an integer $n_0$ such that $I_n = I_{n_0}$, for all $n \geq n_0$.
        \end{lemma}

    \section{Main Results}\label{main_results}
        \begin{definition}[Allowed code space]
              Let $\mathcal{N}$ be a neural network defined on an input space $\mathcal{X} \subseteq \mathbb{R}^d$ that classifies the inputs into $k$-discrete classes. Let $y \in \{1, 2, \ldots, k\}$ denote a specific target class and $\Phi_{\mathcal{N}}: \mathbb{R}^d \to \mathbb{F}_2^m$ denote the threshold map associated with the $m$ neurons of $\mathcal{N}$. The set 
            \begin{equation}
                C(y) = \{ \mathbf{x} \in \mathbb{F}_2^m \mid \Phi_{\mathcal{N}}(\mathbf{v}) = \mathbf{x}, \text{ for some } \mathbf{v} \in \mathcal{X} \}
            \end{equation}
            is the \textit{allowed code space} of $\mathcal{N}$ corresponding to the class $y$.
          \end{definition}
          \begin{theorem}\label{main_theorem}
            Let $\mathcal{N}: \mathbb{R}^d \to \mathbb{R}^k$ be a trained artificial feedforward neural network defined to classify into $k$ discrete classes. Let $y$ be a class among the $k$ discrete classes and $\mathrm{J}_{\mathcal{N}}(y)$ be the neural ideal in the quotient ring $\mathbb{F}_2[x_1, x_2, \ldots, x_m]/\langle x_i^2 - x_i \mid 1 \leq i \leq m \rangle$, where $m$ is number of neurons in $\mathcal{N}$. Let $G_y = \{g_1, g_2, \ldots, g_l\}$ be the generating set of $\mathrm{J}_{\mathcal{N}(y)}$. Furthermore, let $\Phi : \mathbb{R}^d \to \mathbb{F}_2^m$ be the function that maps the input vector $v \in \mathbb{R}^d$ to a code word $x \in \mathbb{F}_2^m$. Given an input vector $v_0 \in \mathbb{R}^d$ and $x_0$ be its corresponding code word, then $\mathcal{N}(v_0) = y$ if and only if $g(x_0) = 0$, for all $g \in G_y$.
          \end{theorem}
          \begin{proof}
            Let $v_0 \in \mathbb{R}^d$ and $x_0$ be its corresponding neural code obtained by the threshold map $\Phi$. Due to definition of the allowed code space, we have $x_0 \in C(y)$. The neural ideal $\mathrm{J}_{\mathcal{N}}(y)$ contains all the polynomial $f \in \mathrm{R}$ such that $f(x) = 0$, for all $x \in C(y)$. Since $G_y \subseteq \mathrm{J}_{\mathcal{N}}(y)$, it follows that for every $x \in C(y)$, $g_i(x) = 0$, for $i = 1, 2, \ldots, l$. In particular for $x = x_0$, we obtain $g_i(x_0) = 0$, for $i = 1, 2, \ldots, l$. To prove the converse, let $v_0 \in \mathbb{R}^d$ be the input vector that maps to an element $x_0$ of $\mathbb{F}_2^m$ satisfying
            \begin{equation}
              g(x_0) = 0, \ \forall \ g \in G_y.
            \end{equation}
            Suppose that the neural network does not classify the targeted class i.e. $\mathcal{N}(v_0) = z$ where $z \neq y$. By the definition of the allowed code space, we have $x_0 \notin C(y)$ which implies that $x_0 \in C(y)^c$. As the class specific neural ideal $\mathrm{J}_{\mathcal{N}}(y)$ is generated by the set consists of unique pseudo monomials $m_{\sigma}$ for every forbidden state $\sigma \in C(y)^c$ and field equations, it implies that there exists a unique pseudo monomial $m_{x_0}(x)$ contained in the generating set of $\mathrm{J}_{\mathcal{N}}(y)$ obtained by $x_0$ and evaluating at $x = x_0$, we get
            \begin{equation}\label{evaluating at x_0}
              m_{x_0}(x_0) = \left(\Pi_{l: x_l = 1}1\right)\left(\Pi_{o:x_o = 0}(1 - 0)\right) = 1
            \end{equation}
           $m_{x_0}(x)$ belongs to the ideal $\mathrm{J}_{\mathcal{N}}(y)$ and $G_y$ is a generating set of $\mathrm{J}_{\mathcal{N}}(y)$ which means that $m_{x_0}(x)$ can be expressed as a linear combination of elements of $G_y$ with polynomial coefficients $h_i(x) \in \mathrm{R}$.
            \begin{equation}\label{m_{x_0}_equation}
              m_{x_0}(x) = \sum_{i = 1}^{l}h_i(x)\cdot g_i(x)
            \end{equation}
            Substituting $x = x_0$ in equation \eqref{m_{x_0}_equation}, we obtain $m_{x_0}(x_0) = \sum_{i = 1}^{l}h_i(x_0)\cdot g_i(x_0)$. By hypothesis statement $g_i(x_0) = 0$, for $i = 1, 2, \ldots, l$, equation \eqref{evaluating at x_0} and the equation \eqref{m_{x_0}_equation} leads to a contradiction to the assumption that $x_0 \in C(y)^c$. Therefore, our assumption is not correct and $x_0 \in C(y)$ finishing the argument.
          \end{proof}
    
            Theorem~\ref{main_theorem} implies that whenever a neural ideal corresponding to a specific class exists, its generating set characterizes whether a given input vector is classified into that class by the neural network. However, we do not have a concrete algorithm for determining the neural ideal explicitly. In an attempt to study this idea, we introduce few definitions and subsequently present an algorithm applicable to certain specific classes of neural networks.
        
          \begin{definition}[Threshold Map]
              Let $\mathcal{N}$ be a neural network defined on an input space $\mathcal{X} \subseteq \mathbb{R}^d$. A \textit{threshold map} associated with the neural network $\mathcal{N}$ is a function $\Phi_{\mathcal{N}}:\mathbb{R}^d \to \mathbb{F}_2^s$ that maps each input vector $v \in \mathbb{R}^d$ to a binary discrete vector in $\mathbb{F}_2^s$, where $s$ is the number of neurons in the neural network $\mathcal{N}$.
          \end{definition}

          \subsection{Proposed Algorithms}
            The only assumption required for the algorithm is that the neural network accepts binary inputs, namely (0) or (1).
        
          \subsubsection{Algorithm : Threshold Map of an Artificial Feedforward Neural Network}\label{algo_threshold}
          \begin{enumerate}
            \item Start with an artificial feedforward neural network $\mathcal{N}$.
            \item For a given layer $L$ with $s$ neurons, and for each neuron in this layer, obtain its bias value $b_j$, for $1 \leq j \leq s$.
            \item Compute the threshold values $\theta_j$ of the layer by $\theta_j = -b_j$.
            \item Define the threshold map $\Phi_{\mathcal{N}, L}: \mathbb{R}^s \to \mathbb{F}_2^s$ of the layer $L$ by the following coordinate map
            \begin{equation}
              \Phi_{\mathcal{N}, L}(v)_j = \begin{cases} 1 & \text{if } \sum_{k} W_{j,k}v_k \geq \theta_j \\ 0 & \text{otherwise} \end{cases}
            \end{equation}, for $1 \leq j \leq s$.
            \item Concatenate the threshold maps of each layer $\Phi_{\mathcal{N}, L}$ to obtain the threshold map $\Phi_{\mathcal{N}}$.
          \end{enumerate}
        
          \subsubsection{Algorithm : Neural code of an input vector of an Artificial Feedforward Neural Network}\label{algo_neural_code}
          \begin{enumerate}
            \item Start with an artificial feedforward neural network $\mathcal{N}$.
            \item Compute the threshold map $\Phi$ using Algorithm \ref{algo_threshold}.
            \item For an input vector $v$, compute the output vector of each layer $L$, denoted by $o_L(v)$.
            \item Concatenate the output vectors of each layer to obtain output vector
            \begin{equation}
                O(v) = [o_{L_1}(v), o_{L_2}(v), \dots, o_{L_m}(v)]^T,
            \end{equation}
            where $m$ denotes the number of neurons in $\mathcal{N}$.
            \item Compute the neural code for the input vector $v$ by evaluating the threshold map at $O(v)$.
          \end{enumerate}

          \subsubsection{Algorithm for Computing the Allowed Code Space of a Class}\label{alg:allowed_code_space}
        
          \begin{enumerate}
              \item Start with an artificial feedforward neural network $\mathcal{N}$ and a set $\Omega$ consisting of all vectors from the training data whose target class is $y$, where $y$ is one of the classes that $\mathcal{N}$ classifies.
            \item Compute the threshold map $\Phi$ for $\mathcal{N}$ using Algorithm \ref{algo_threshold}.
            \item For each input vector $v \in \Omega$, compute its corresponding neural code by using Algorithm \ref{algo_neural_code}.
            \item Collect all the computed neural codes into a single set.
            \item Define this set as the allowed code space for class $y$.
          \end{enumerate}
          
          For a given neural network and a target class, Algorithm~\ref{alg:allowed_code_space} is used to determine the corresponding allowed code space. From the Algorithm in \cite{Curto2013}, we compute the canonical form of the neural ideal associated with the target class classified by the neural network.
          
          \begin{remark}
              The above algorithms can be generalized to arbitrary neural network architectures by treating the first hidden layer as the initial layer of analysis rather than the input layer. While this approach does not allow direct determination of input ideal memberships. However, this approach enables us to interpret the features learned by individual hidden neurons within the network. 
          \end{remark}
          The computational complexity of the proposed method depends on two primary factors: the cardinality of the data points and the total number of neurons within the network architecture. We focus exclusively on addressing the constraints of the first case to compute the corresponding neural ideals. While the second case falls outside the scope of this study and is left for future investigation.
    
          \begin{lemma}\label{lem:ideal_containment}
            Let $\mathcal{N}: \mathbb{R}^d \to \mathbb{R}^k$ be a trained artificial feedforward neural network, and let $\Phi : \mathbb{R}^d \to \mathbb{F}_2^m$ be its associated threshold map. Let $\Omega \subseteq \mathbb{R}^d$ be the input space corresponding to a target class $y$. For any subset of inputs $D \subseteq \Omega$, let $\mathcal{A}(D)$ and $\mathcal{F}(D)$ denote the set of allowed and forbidden neural codes, respectively. If $D_1$ and $D_2$ are two subsets of inputs such that $D_1 \subseteq D_2 \subseteq \Omega$, then $\mathrm{J}_{\mathcal{N}_{D_2}}(y) \subseteq \mathrm{J}_{\mathcal{N}_{D_1}}(y)$.
          \end{lemma}
            
          \begin{proof}
            Let $D_1 \subseteq D_2 \subseteq \Omega$. By the definition of the allowed code space, we have $\mathcal{A}(D_1) \subseteq \mathcal{A}(D_2)$.
            By taking the complements in $\mathbb{F}_2^m$ to find the forbidden states, we obtain $\mathcal{F}(D_2) \subseteq \mathcal{F}(D_1)$.
            Since $\mathcal{F}(D_2)$ is a subset of $\mathcal{F}(D_1)$, a polynomial $f$ that vanishes on all elements of $\mathcal{F}(D_1)$ must necessarily vanish on all elements of $\mathcal{F}(D_2)$. Therefore, $ \mathrm{J}_{\mathcal{N}_{D_2}}(y) \subseteq \mathrm{J}_{\mathcal{N}_{D_1}}(y)$.
            This completes our argument.
            \end{proof}
          \begin{theorem}
              Let $\mathcal{N}: \mathbb{R}^d \to \mathbb{R}^k$ be a trained artificial feedforward neural network defined to classify into $k$ discrete classes. Let $\Phi : \mathbb{R}^d \to \mathbb{F}_2^m$ be the function that maps the input vector $v \in \mathbb{R}^d$ to a neural code $x \in \mathbb{F}_2^m$, where $m$ is the number of neurons in $\mathcal{N}$. Let $G_y = \{g_1, g_2, \ldots, g_l\}$ be the generating set of $\mathrm{J}_{\mathcal{N}(y)}$. Let $\Omega \subseteq \mathbb{R}^d$ be the space of inputs of class $y$ that are used to train the neural network $\mathcal{N}$ containing very large amount of elements. Furthermore, let $y$ be a class among the $k$ discrete classes and $\mathrm{J}_{\mathcal{N}_{\Omega}}(y)$ be the neural ideal of $\mathcal{N}$. Then there exists a natural number $N_0$ such that for every $\Omega_0 \subseteq \Omega$ containing at least $N_0$ elements $J_{\mathcal{N}}(y) = J_{\mathcal{N}_{\Omega_0}}(y)$, where $J_{\mathcal{N}_{\Omega_0}}(y)$ is the neural ideal of $\mathcal{N}$ constructed only with the elements of $\Omega_0$.
          \end{theorem}
          \begin{proof}
              Let $\Delta_1 \subseteq \Delta_2\subseteq \cdots \subseteq \Omega$ be a sequence of subsets, where $\Delta_i$ consists of $i$ elements. By Lemma \ref{lem:ideal_containment}, this nested sequence induces a nested descending chain of neural ideals in $S = $.
              \begin{equation}\label{eq:des_chain_equation}
                  J_{\mathcal{N}_{\Delta_1}(y)} \supseteq J_{\mathcal{N}_{\Delta_2}(y)} \supseteq J_{\mathcal{N}_{\Delta_3}(y)} \supseteq \cdots \supseteq J_{\mathcal{N}_{\Omega}}(y)
              \end{equation}
    
              Since $\mathbb{F}_2[x_1, x_2, \ldots, x_m]/\langle x_i^2 - x_i \mid 1 \leq i \leq m \rangle$ is finite, where $m$ is number of neurons in $\mathcal{N}$. By Lemma \ref{lem:des_chain_condition_initial}, the descending chain of ideals \eqref{eq:des_chain_equation} must stabilize after a finite number of steps. Thus, there exists a natural number $N_0$ such that $\forall \, n \geq N_0$
              \begin{equation}\label{eq:ideal_stabilization}
                  J_{\Delta_n}(y) = J_{\mathcal{N}_{\Delta_{N_0}}}(y)
              \end{equation}
    
              Let $\Omega_0 \subseteq \Omega$ be an arbitrary subset containing at least $N_0$ elements. By Lemma \ref{lem:ideal_containment}, we have $J_{\mathcal{N}_{\Omega_0}}(y) \supseteq J_{\mathcal{N}_{\Omega}}(y)$. If $\Omega_0$ contains at least $N_0$ elements, from \eqref{eq:ideal_stabilization} we obtain $J_{\mathcal{N}_{\Omega_0}}(y) \subseteq J_{\mathcal{N}_{\Omega_{N_0}}}(y) = J_{\mathcal{N}_{\Omega}}(y)$. Combining both inclusions, we get $J_{\mathcal{N}_{\Omega}}(y) = J_{\mathcal{N}_{\Omega_0}}(y
              )$. This completes the proof.
          \end{proof}
    \section{Applications}\label{sec:applications}
    
        In this section, we illustrate the proposed framework through two classification problems of different scales. We first consider the classical XOR classification problem \cite{3175865}, where the associated neural ideal can be computed explicitly and interpreted in terms of the activation behavior of the network. We then apply the same methodology to the MNIST handwritten digit dataset in order to examine the algebraic structure of hidden representations learned by a larger neural network.
        
        These examples demonstrate how neural codes, allowed code spaces, and neural ideals can be used to provide an algebraic description of neural network classification while retaining information about the internal representations learned by the model.

    \subsection{XOR Classification}
    
        We illustrate the proposed framework using the classical XOR classification problem. Although simple in structure, XOR represents a fundamental non-linearly separable classification task and provides a convenient setting for explicitly constructing neural codes, allowed code spaces, and neural ideals.
        
        Consider a feedforward neural network consisting of five neurons, denoted by $x_1,x_2,\ldots,x_5$, arranged into input, hidden, and output layers. The network architecture together with the corresponding weights and activation thresholds is shown in Figure~\ref{fig:xor_network}.
    
         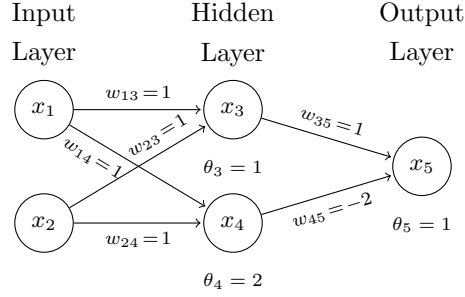
\begin{figure}[htbp]
            \centering
            \begin{tikzpicture}[
                    shorten >=1pt,->,
                    draw=black, node distance=2.5cm,
                    every pin edge/.style={<-,shorten <=1pt},
                    neuron/.style={circle, fill=white, minimum size=22pt, inner sep=0pt, draw=black, font=\small},
                    input neuron/.style={neuron},
                    hidden neuron/.style={neuron},
                    output neuron/.style={neuron},
                    annot/.style={text width=4em, text centered},
                    weight/.style={sloped, allow upside down, font=\scriptsize}
                ]
            
                    \foreach \name / \y in {1/1, 2/2}
                        \node[input neuron] (I-\name) at (0,-\y*1.5) {$x_\name$};
                
                    \foreach \name / \y in {3/1, 4/2}
                        \node[hidden neuron] (H-\name) at (2.5cm,-\y*1.5) {$x_\name$};
                
                    \node[output neuron] (O) at (5cm,-2.25) {$x_5$};
                
                    \path (I-1) edge node[weight, above, midway] {$w_{13}\!=\!1$} (H-3);
                    \path (I-1) edge node[weight, below, pos=0.23]{$w_{14}\!=\!1$} (H-4); 
                    \path (I-2) edge node[weight, above, pos=0.7] {$w_{23}\!=\!1$} (H-3); 
                    \path (I-2) edge node[weight, below, midway] {$w_{24}\!=\!1$} (H-4);
                
                    \path (H-3) edge node[weight, above, midway] {$w_{35}\!=\!1$} (O);
                    \path (H-4) edge node[weight, below, midway] {$w_{45}\!=\!-2$} (O);
                
                    \node[below=0.15cm of H-3, font=\scriptsize] {$\theta_3 = 1$};
                    \node[below=0.15cm of H-4, font=\scriptsize] {$\theta_4 = 2$};
                    \node[below=0.15cm of O, font=\scriptsize] {$\theta_5 = 1$};
                
                    \node[annot, above of=I-1, node distance=1cm] {Input\\Layer};
                    \node[annot, above of=H-3, node distance=1cm] {Hidden\\Layer};
                    \node[annot, above of=O, node distance=1.75cm] {Output\\Layer};
            \end{tikzpicture}
            \caption{The 5-neuron feedforward network layout optimized for XOR classification, including corresponding weights ($w_{ij}$) and activation thresholds ($\theta_i$).}
            \label{fig:xor_network}
            \end{figure}

    It is straightforward to verify that the network realizes the XOR classification function. We now analyze the same network through the algebraic framework developed in the previous section. The objective is to determine the neural codes generated by the network, construct the corresponding allowed code space, and derive the associated neural ideal.
    
    Since the output neuron is determined by the activation patterns of the input and hidden neurons, we focus on the first four neurons. Each neuron can either be active or inactive, and hence every activation pattern can be represented by a binary vector in $\mathbb{F}_2^4$.
    
    Using Algorithm~\ref{algo_threshold} and Algorithm~\ref{algo_neural_code}, we first compute the neural code associated with each input vector. To illustrate the procedure, consider the input vector $(1,0)$.
    
    \begin{enumerate}
        \item For the neuron $x_3$,
        \[
        w_{13}x_1+w_{23}x_2 = 1 \geq \theta_3 =1.
        \]
        Hence $x_3$ activates and therefore $x_3=1$.
    
        \item For the neuron $x_4$,
        \[
        w_{14}x_1+w_{24}x_2 = 1 < \theta_4 =2.
        \]
        Hence $x_4$ does not activate and therefore $x_4=0$.
    \end{enumerate}
    
    Consequently, the neural code associated with the input $(1,0)$ is $(1,0,1,0)$. Repeating the same computation for the remaining input configurations yields the neural codes listed in Table~\ref{xor_neural_code_table}.
        \begin{table}[htbp]
            \centering
            \begin{tabular}{|>{\centering\arraybackslash}p{1.2cm}|c|>{\centering\arraybackslash}p{2cm}|}
                \hline
                Input & Class & Neural code\\
                \hline
                $(0, 0)$ & 0 & $(0, 0, 0, 0)$\\
                \hline
                $(1, 0)$ & 1 & $(1, 0, 1, 0)$\\
                \hline
                $(0, 1)$ & 1 & $(0, 1, 1, 0)$\\
                \hline
                $(1, 1)$ & 0 & $(1, 1, 1, 1)$\\
                \hline
            \end{tabular}
            \caption{Neural codes of XOR}
            \label{xor_neural_code_table}
        \end{table}
    We next determine the allowed code space corresponding to the positive XOR class. By Definition~3.1, the allowed code space consists of all neural codes associated with inputs classified into Class $\mathbf{1}$. From Table~\ref{xor_neural_code_table}, we obtain
    \[
    C(\mathbf{1})
    =
    \{
    (1,0,1,0),
    (0,1,1,0)
    \}.
    \]
    
    The corresponding neural ideal is constructed from the forbidden states, namely those binary vectors that do not belong to the allowed code space. Since $\mathbb{F}_2^4$ contains $2^4=16$ binary vectors and only two belong to $C(\mathbf{1})$, the remaining fourteen vectors constitute the forbidden code space
    
    \[
    C(\mathbf{1})^c
    =
    \mathbb{F}_2^4
    \setminus
    C(\mathbf{1}).
    \]
    
    For each forbidden vector
    $\mathbf a=(a_1,a_2,a_3,a_4)$,
    we associate the pseudo-monomial
    
    \[
    \prod_{i=1}^{4}(1-a_i-x_i).
    \]
    
    Rather than working with all forbidden states individually, we identify the algebraic relations satisfied by every neural code belonging to $C(\mathbf{1})$. Since all computations take place over $\mathbb{F}_2$, the Boolean relations $x_i^2=x_i$ hold for each $i\in\{1,2,3,4\}$.
    
    The neural codes belonging to the class $\mathbf{1}$ satisfy the following structural constraints.
    
    \begin{enumerate}
        \item For every
        $a=(a_1,a_2,a_3,a_4)\in C(\mathbf{1})$,
        we have $a_4=0$.
        Hence every state with $a_4=1$ is forbidden, giving rise to the polynomial relation $x_4$.
    
        \item For every
        $a=(a_1,a_2,a_3,a_4)\in C(\mathbf{1})$,
        we have $a_3=1$.
        Hence every state with $a_3=0$ is forbidden, yielding the relation $1-x_3$.
    
        \item For every
        $a=(a_1,a_2,a_3,a_4)\in C(\mathbf{1})$,
        the coordinates $a_1$ and $a_2$ are distinct. Consequently, the following configurations are forbidden:
        \begin{enumerate}
            \item $a_1=a_2=0$, yielding the relation
            \[
            (1-x_1)(1-x_2),
            \]
    
            \item $a_1=a_2=1$, yielding the relation
            \[
            x_1x_2.
            \]
        \end{enumerate}
    
        Since $x_1x_2=0$, the relation
        \[
        (1-x_1)(1-x_2)
        \]
        simplifies to
        \[
        1-x_1-x_2.
        \]
    \end{enumerate}
    
    Combining these relations together with the Boolean relations
    $x_i^2-x_i$,
    for
    $i\in\{1,2,3,4\}$,
    we obtain the neural ideal
    
    \[
    J_{\mathcal N}(\mathbf 1)
    =
    \left\langle
    1-x_1-x_2,\;
    x_1x_2,\;
    1-x_3,\;
    x_4
    \right\rangle.
    \]
    
    The neural ideal provides a complete algebraic description of the XOR classification mechanism. Each generator corresponds to a structural constraint satisfied by every admissible neural code associated with the target class.
    
    \begin{enumerate}
        \item The generator $x_4$ indicates that the AND feature represented by the neuron $x_4$ must remain inactive.
    
        \item The generator $1-x_3$ implies that the OR feature represented by the neuron $x_3$ must remain active.
    
        \item The generator $1-x_1-x_2$ encodes the defining XOR condition that exactly one of the two input neurons is active.
    \end{enumerate}
    
    Thus, the classification behaviour of the network can be recovered directly from the algebraic relations contained in the neural ideal.

    \subsection{Feature Interpretation through Neural Codes}
    
    One of the important aspects of the proposed algebraic framework is its ability to provide insight into the internal representations learned by a neural network. Although neural networks are often regarded as black-box models, the neural codes introduced in this work provide a structured description of hidden-layer activations. By associating activation patterns with the corresponding input samples, it becomes possible to investigate the features captured by individual neurons and to obtain an interpretable description of their functional roles within the network.
    
    The central idea is that neurons exhibiting similar activation behaviour correspond to specific visual or structural patterns present in the input data. Consequently, the combinatorial information encoded by neural codes can be used not only for classification analysis but also for understanding the features learned by the network. The following procedure describes the feature interpretation framework employed in this work.
    
    \subsubsection{Algorithm: Feature Visualization through Neural Codes}
    
    \begin{enumerate}
        \item Begin with a trained feedforward neural network together with its training dataset.
        
        \item Compute the threshold map using Algorithm~\ref{algo_threshold}.
        
        \item Apply Algorithm~\ref{algo_neural_code} to determine the neural codes associated with the input samples.
        
        \item Select a hidden neuron and identify all input samples for which the neuron remains active.
        
        \item Collect the corresponding input images and compute their average. The resulting image provides a visual representation of the feature associated with the selected neuron.
        
        \item Similarly, compute the average of all input images for which the neuron remains inactive. This image represents the complementary suppression feature associated with the neuron.
        
        \item Compare the activation and suppression patterns to obtain an interpretable description of the feature learned by the neuron.
    \end{enumerate}
    
    
    \subsection{MNIST Handwritten Digit Classification}
    
    To demonstrate the applicability of the proposed framework on a practical classification task, we consider a feedforward neural network trained on the MNIST handwritten digit dataset \cite{lecun-mnisthandwrittendigit-2010}. The network maps input images represented as vectors in $\mathbb{R}^{784}$ to one of the ten digit classes.
    
    The neural network consists of the following layers.
    
    \begin{enumerate}
        \item A first hidden layer containing $64$ neurons followed by a ReLU activation.
        
        \item A second hidden layer containing $16$ neurons followed by a ReLU activation.
        
        \item An output layer consisting of $10$ neurons responsible for the final classification.
    \end{enumerate}
    
    The model is trained using the Adam optimizer for ten epochs. Following training, the activation thresholds associated with the hidden neurons are determined. For a neuron with bias parameter $b$, the corresponding threshold is given by $-b$, which represents the hyperplane at which the activation state of the neuron changes.
    
    Using Algorithm~\ref{alg:allowed_code_space}, each input image is transformed into a binary activation pattern, referred to as its neural code. Images belonging to the same class tend to generate similar neural codes, indicating that the network organizes visually related inputs into common combinatorial regions of the hidden representation space. The collection of neural codes corresponding to a particular class forms the associated allowed code space.
    
    For a fixed target class, such as the digit $3$, let $C_3$ denote the set of neural codes generated by images belonging to that class. The complement of $C_3$ within the space of all binary activation states yields the corresponding forbidden code space, denoted by $\mathcal{F}$. Each forbidden code
    \[
    \mathbf{a}=(a_1,\ldots,a_m)
    \]
    determines the pseudo-monomial
    \[
    \prod_{k=1}^{m}
    \left(
    a_kx_k+(1-a_k)(1-x_k)
    \right).
    \]
    
    The neural ideal associated with the target class is generated by these pseudo-monomials. Consequently, the neural ideal provides an algebraic description of the combinatorial constraints satisfied by the hidden representations of that class.
    
    While neural ideals characterize hidden representations algebraically, understanding their practical significance requires relating hidden neurons back to the image domain. To accomplish this, we employ the feature interpretation procedure introduced in the previous subsection.
    
    For each hidden neuron, all training images that activate the neuron are collected and averaged. The resulting image represents the visual structure most strongly associated with that neuron. Similarly, averaging the images for which the neuron remains inactive produces a complementary suppression feature. The difference between these two averages yields a feature balance map that highlights image regions contributing positively or negatively to neuron activation. These reconstructed visualizations provide an empirical interpretation of the features captured by the hidden neurons.
    
    To further investigate how the learned hidden representations contribute to the final prediction, we analyze the weight matrix connecting the hidden neurons to the output layer. Figure~\ref{fig:features_and_weights} presents the corresponding heatmap, where positive weights indicate evidence supporting a class, whereas negative weights indicate evidence opposing that class.
    
    \begin{figure}[ht]
        \centering
        
        \begin{subfigure}{0.48\textwidth}
            \centering
            \includegraphics[width=\linewidth]{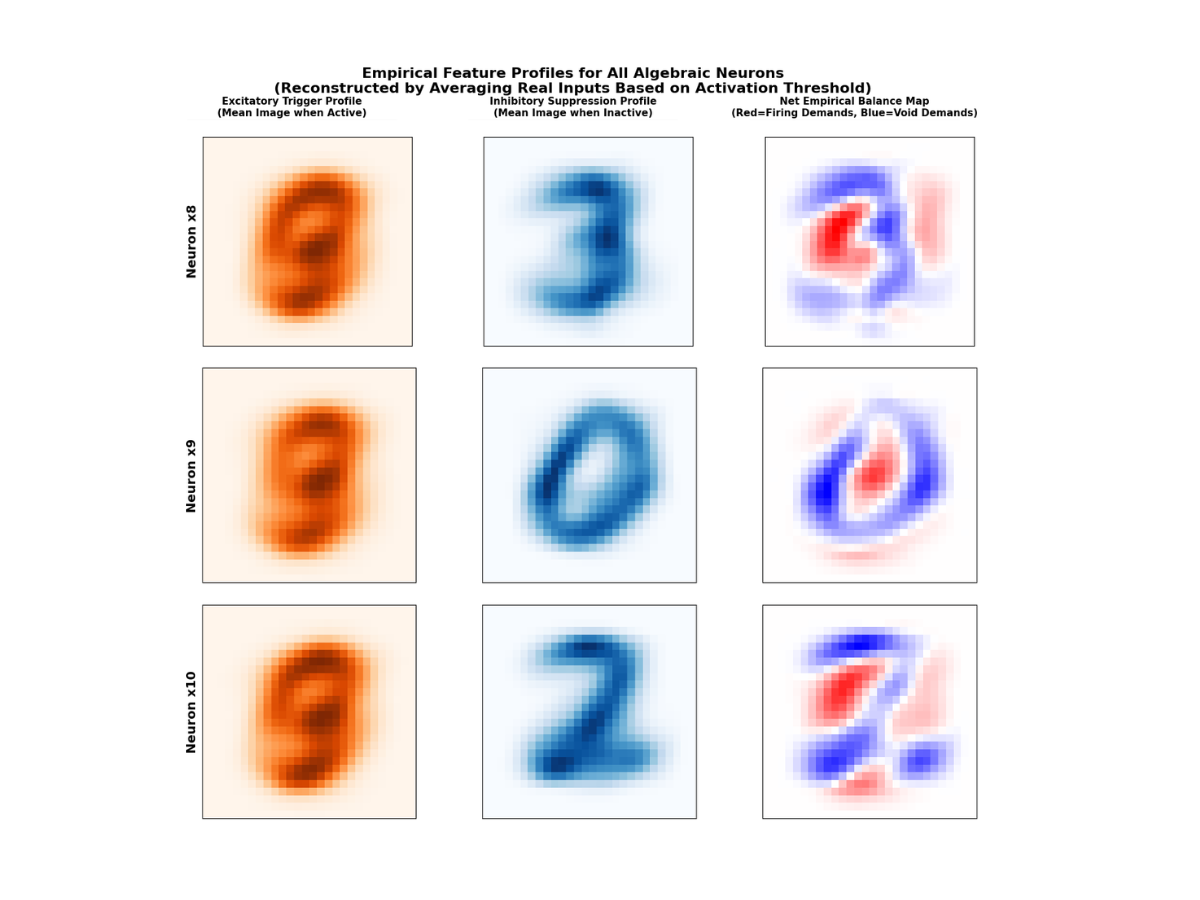}
            \caption{Features learned by hidden neurons 8, 9, 10.}
            \label{fig:learned_features_of_hidden_neurons}
        \end{subfigure}
        \hfill
        \begin{subfigure}{0.48\textwidth}
            \centering
            \includegraphics[width=\linewidth]{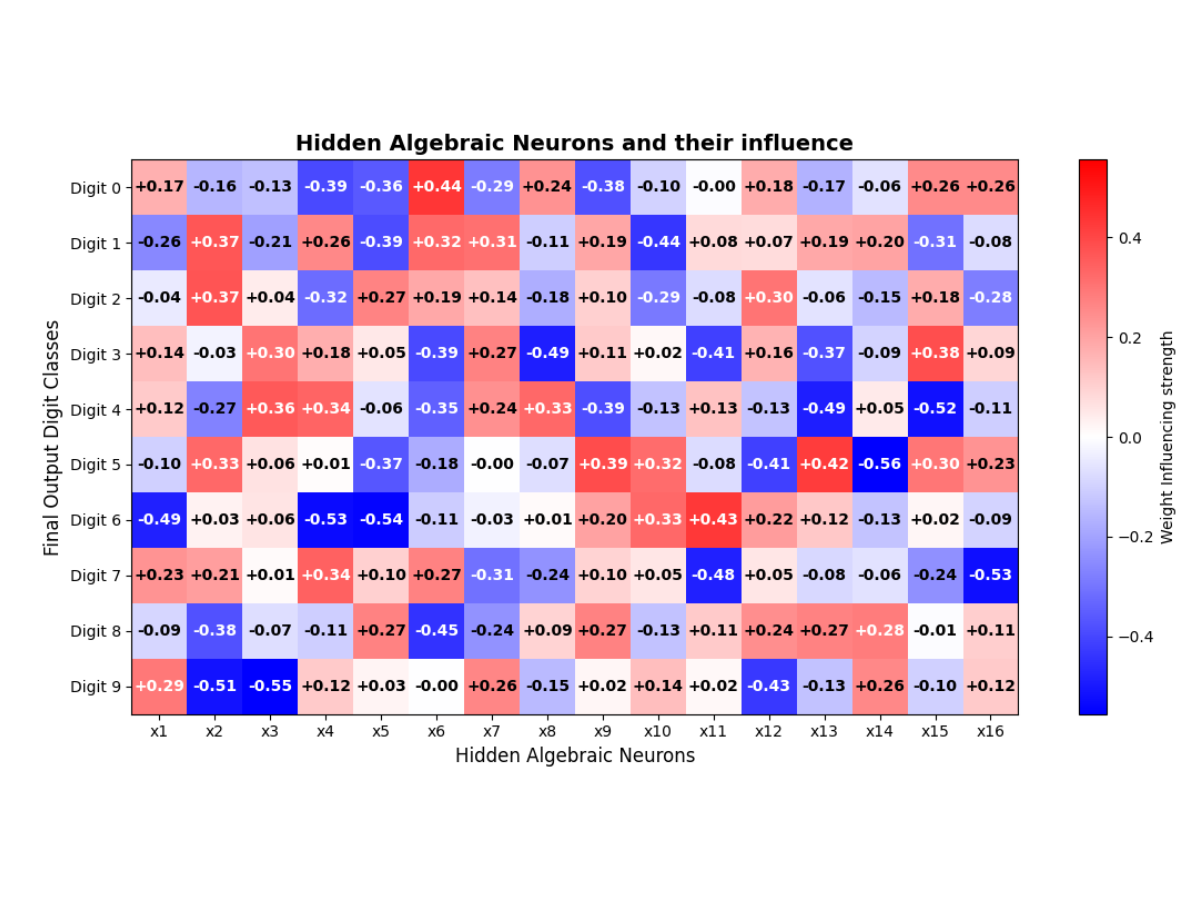}
            \caption{Weight matrix connecting hidden neurons to the output layer.}
            \label{fig:learned_weight_matrix}
        \end{subfigure}
        
        \caption{Visualization of learned features and output-layer weights.}
        \label{fig:features_and_weights}
    \end{figure}
    
    The heatmap indicates that the classifier does not rely on any single hidden neuron. Instead, each digit class is characterized by a distinct combination of supporting and suppressing features distributed across multiple neurons. The final classification decision therefore emerges from the collective contribution of the hidden representations learned by the network.

    \begin{itemize}
    
        \item \textbf{Digit 0.}
        The weight matrix indicates that neurons $x_6$ and $x_8$ contribute positively towards the classification of digit $0$, whereas neurons $x_4$ and $x_5$ contribute negatively. This suggests that the features associated with $x_6$ and $x_8$ are more frequently observed in images of digit $0$, while the patterns represented by $x_4$ and $x_5$ are less compatible with this class.
    
        \item \textbf{Digit 1.}
        Digit $1$ receives strong positive contributions from neurons $x_2$, $x_6$, and $x_7$, whereas neuron $x_5$ is strongly suppressed. This behaviour indicates that the hidden representations encoded by $x_2$, $x_6$, and $x_7$ play an important role in distinguishing images belonging to this class.
    
        \item \textbf{Digit 3.}
        The classification of digit $3$ exhibits a strong positive dependence on neuron $x_3$, while neuron $x_8$ contributes negatively. This observation suggests that the representation of digit $3$ is determined by a combination of supporting and suppressing features rather than by the activation of a single neuron.
    
        \item \textbf{Digit 6.}
        Digit $6$ is associated with some of the strongest negative weights in the entire matrix, particularly for neurons $x_1$, $x_4$, and $x_5$. At the same time, neurons $x_2$ and $x_3$ contribute positively. This indicates that the network identifies digit $6$ through a selective balance between favourable and unfavourable hidden representations.
    
        \item \textbf{Digit 9.}
        The classification of digit $9$ is characterized by negative contributions from neurons $x_2$ and $x_3$, together with positive support from neuron $x_1$. Such behaviour illustrates how the network differentiates visually similar digit classes through variations in their hidden activation patterns.

    \end{itemize}
    \subsection{Software Implementation and Interactive Exploration Framework}
    
    To demonstrate the practical applicability of the proposed framework, we developed an interactive software tool, namely the \emph{Neural Network Representation Explorer}. The tool implements the complete pipeline introduced in this work, including threshold-map construction, neural-code generation, allowed code space computation, neural ideal analysis, and feature interpretation.
    
    The primary objective of the tool is to provide an interpretable interface for studying the internal representations learned by trained neural networks. Given a trained model and the corresponding dataset, the framework automatically computes the binary activation patterns associated with hidden neurons and constructs the corresponding neural-code representation. These codes are subsequently used to identify class-specific allowed code spaces and their associated algebraic structures.
    
    Figure~\ref{fig:tool_overview} presents the graphical interface of the developed system. The left panel visualizes the neural network architecture together with the learned inter-layer connections. Positive and negative weights are displayed using distinct colour schemes, allowing the contribution of individual neurons to be examined visually. The right panel provides detailed information about the selected output node, including the activation state, threshold metric, allowed code-space information, and receptive-field relationships.
    
    \begin{figure*}[t!]
    \centering
    \includegraphics[width=0.85\textwidth]{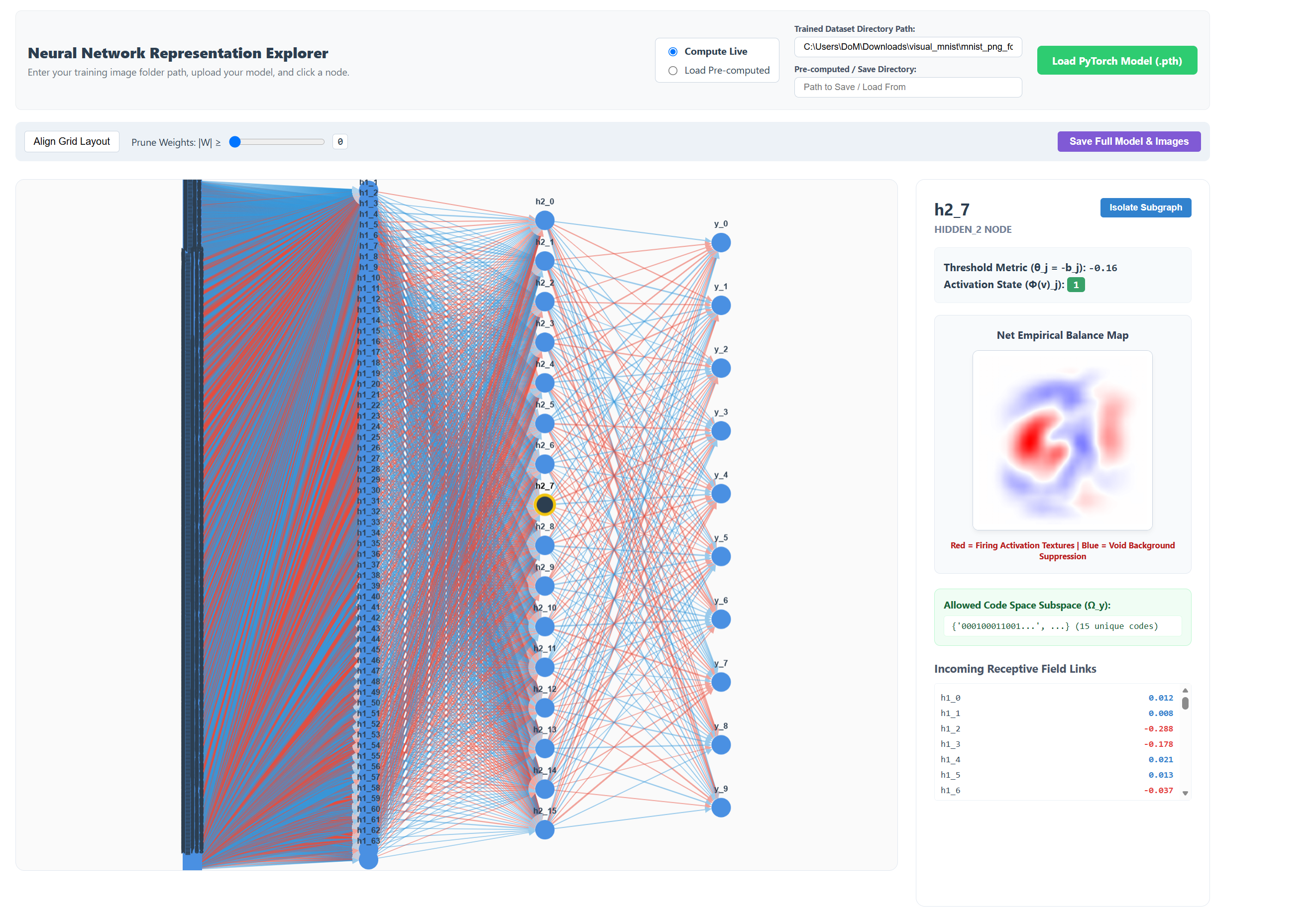}
    \caption{Graphical interface of the Neural Network Representation Explorer implementing the proposed framework. The interface visualizes network connectivity, activation states, threshold-map information, allowed code spaces, and class-specific feature representations.}
    \label{fig:tool_overview}
    \end{figure*}
    
    An important component of the framework is the construction of feature balance maps. For a selected neuron or output class, the tool computes the average activation pattern of the corresponding input samples and contrasts it with the average pattern of suppressed samples. The resulting visualization highlights image regions that contribute positively or negatively to neuron activation. Consequently, the learned representations can be interpreted directly in the input domain.
    
    Figure~\ref{fig:digit_feature_maps} illustrates representative feature balance maps obtained for several output classes of the MNIST dataset. Although the network is trained solely for classification, the generated visualizations reveal class-dependent structural patterns learned by the hidden layers. Distinct activation regions can be observed for different digit classes, indicating that the neural-code representation preserves meaningful information regarding the internal organization of the network.
    
    \begin{figure}[H]
        \centering
        \includegraphics[width=14cm]{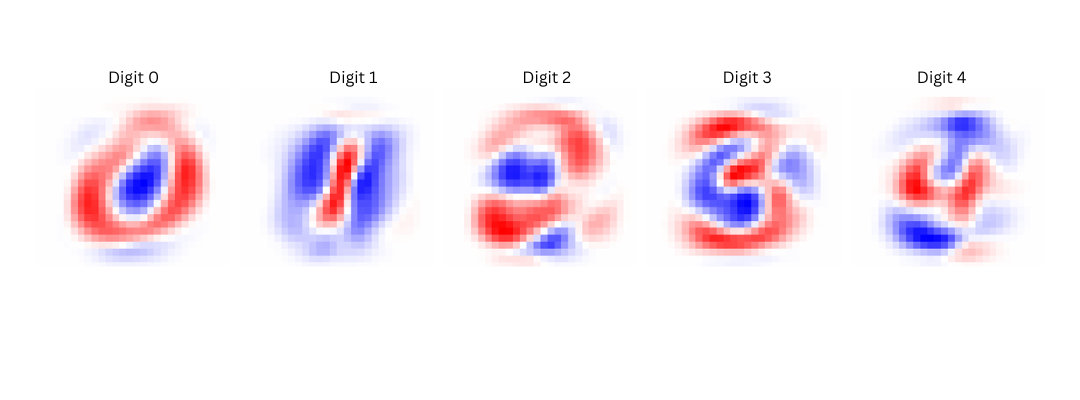}
        \caption{Feature balance maps generated using the proposed framework for different digit classes. Red regions indicate activation-supporting structures, whereas blue regions correspond to suppressive regions. The visualizations provide an interpretable description of the representations learned by the network.}
        \label{fig:digit_feature_maps}
    \end{figure}
    
    The developed software serves as an implementation of the theoretical framework proposed in this paper. Beyond classification analysis, it enables interactive exploration of neural codes, activation patterns, class-specific code spaces, and learned representations. This demonstrates that the proposed algebraic formulation can be translated into a practical tool for investigating the behaviour of trained neural networks.
    
    However, this approach requires a significant amount of computational power to represent larger neural networks.

    \section{Conclusion}
    
    This work introduced an algebraic framework for analyzing neural network representations using neural codes and neural ideals. By constructing class-specific allowed code spaces from hidden-layer activation patterns, we established a direct connection between neural network classification and Boolean polynomial relations. The proposed framework provides a mathematical description of the decision-making process while simultaneously enabling the interpretation of hidden representations through neural codes.
    
    In addition to the theoretical results, we developed algorithms for threshold-map construction, neural-code generation, allowed code-space computation, and neural-ideal derivation. A stabilization result was established to guarantee the invariance of neural ideals obtained from sufficiently large collections of training samples. The XOR and MNIST experiments demonstrated that the framework can capture both classification behaviour and feature-level representations in a unified algebraic setting.
    
    To facilitate practical exploration, an interactive software framework was developed for visualizing network connectivity, activation patterns, allowed code spaces, neural ideals, and feature balance maps. The software illustrates how the proposed mathematical framework can be translated into an interpretable analysis tool for trained neural networks. The presented results indicate that neural ideals offer a promising direction for studying representation learning, interpretability, and the internal structure of neural network models.
    
    \section*{Data availability}
    
    The data used in this study are publicly available dataset. The implementation code and experiment scripts are maintained in a GitHub repository at \url{ https://github.com/yvs1967/neural-network-representation-explorer/}. 

    \bibliography{ref}

\end{document}